\documentclass[twoside]{article}
\pdfoutput=1
\usepackage[letterpaper=true,pdfview=FitH,pdfstartview=FitH,bookmarks=false]{hyperref}
\usepackage[round]{natbib}
\usepackage[accepted]{aistats2e}

\usepackage{color}
\usepackage{algorithm}
\usepackage{algorithmic}
\usepackage{MathLib}
\usepackage{RevEng-LaTeX-Lib}
\usepackage{amsmath}
\usepackage{textcomp}
\usepackage{amsthm}

\newtheorem{theorem}{Theorem}[section]

\newtheorem{lemma}[theorem]{Lemma}
\newtheorem{program}[theorem]{Program}
\newenvironment{Program}[1]
  {\begin{program} \quad #1 \vspace{0mm} \\ \mbox{} \quad \quad \textbf{s.t.} \begin{minipage}[c]{.8\textwidth} \begin{trivlist}}
  {\end{trivlist} \end{minipage} \end{program}}

\theoremstyle{definition}
\newtheorem{algo}[theorem]{Algorithm}
\newenvironment{Algorithm}[2]
{\begin{algo} \label{#2} \quad {#1}  \\ \small \rule[1cm]{\linewidth}{.5pt}\vspace{-15mm}
    \begin{algorithmic}}
{\end{algorithmic}\vspace{-3mm}\rule{\linewidth}{.5pt}\end{algo}}

\newcommand{\ie}{\emph{i.e.},}
\newcommand{\eg}{\emph{e.g.},}

\newcommand{\vpar}{\vspace*{.3em}}
\newcommand\stitle[1]{\vpar\noindent{\bf #1\/}}
\newcommand{\term}[1]{\emph{#1}}
\newcommand{\nth}[2]{\ensuremath{{#1}^{\mbox{\scriptsize #2}}}} 
\newcommand{\mat}[1]{\ensuremath{\mathbf{#1}}}
\renewcommand{\vec}[1]{\ensuremath{\mathbf{#1}}}
\newcommand{\bigO}[1]{\mathcal{O}\left(#1\right)}
\newcommand{\bigOmega}[1]{\Omega\left(#1\right)}
\newcommand{\bigOstar}[1]{\mathcal{O}^\ast\left(#1\right)}
\newcommand{\dims}{D}
\newcommand{\negLbl}{\textrm{\textquotesingle$-$\textquotesingle}}
\newcommand{\posLbl}{\textrm{\textquotesingle+\textquotesingle}}
\newcommand{\ithCost}[1]{\ensuremath{c_{#1}}}
\newcommand{\coordinVect}[1]{\ensuremath{\vec{\delta}_{#1}}}
\newcommand{\aset}[1]{\ensuremath{\mathcal{#1}}}

\newcommand{\minkFunc}[2]{\mathrm{m}_{#1}\left(#2\right)}
\newcommand{\ball}[2]{\ensuremath{\mathcal{B}^{#1}\left(#2\right)}}
\newcommand{\oneball}[2]{\ensuremath{\mathcal{B}^{#1}_{1}\left(#2\right)}}
\newcommand{\LP}[1][p]{\ensuremath{\ell_{#1}}}

\newcommand{\MAC}{\textit{MAC}}
\newcommand{\IMAC}{\textit{IMAC}}
\newcommand{\kIMAC}[1][\epsilon]{\ensuremath{#1}\textit{-\IMAC{}}}
\newcommand{\ACRE}{\textit{ACRE}}
\newcommand{\kACRE}[1][\epsilon]{\ACRE\ \ensuremath{#1}-learnable}
\newcommand{\kSearchable}[1][\epsilon]{\kIMAC[#1] searchable}
\newcommand{\kSearchability}[1][\epsilon]{\kIMAC[#1] searchability}
\newcommand{\convexClass}{convex-inducing classifiers}
\newcommand{\ConvexClass}{Convex-inducing classifiers}
\newcommand{\maxCost}{\ensuremath{C^-}}
\newcommand{\minCost}{\ensuremath{C^+}}

\newcommand{\addGoal}{\ensuremath{\eta}}
\newcommand{\multGoal}{\ensuremath{\epsilon}}

\newcommand{\MLS}[1][K]{MultiLineSearch}
\newcommand{\KMLS}[1][K]{\textsc{\ensuremath{#1}-step MultiLineSearch}}

\newcommand{\inst}[1]{\ensuremath{^{\mbox{\scriptsize {#1}}}}}

\newcommand{\eat}[1]{}
\newcommand{\ban}[1]{}

\begin{document}



\runningauthor{Nelson, Rubinstein, Huang, Joseph, Lau, Lee, Rao, Tran, Tygar}

\twocolumn[

\aistatstitle{Near-Optimal Evasion of Convex-Inducing Classifiers}



\newcommand{\negspace}{\hspace{-2.0em}}

\aistatsauthor{ Blaine Nelson\inst{1} \hspace{-1.0em} \And \hspace{-1.0em} Benjamin I. P. Rubinstein\inst{1} \hspace{-2.0em} \And \hspace{-2.0em} Ling Huang\inst{2} \hspace{-2.0em} \And \hspace{-2.0em} Anthony D. Joseph\inst{1,2} \\[.1em]} 

\aistatsauthor{ Shing-hon Lau\inst{3} \hspace{-2.0em} \And \hspace{-2.0em} Steven J. Lee\inst{1} \hspace{-2.0em} \And \hspace{-2.0em} Satish Rao\inst{1} \hspace{-2.0em} \And \hspace{-2.0em} Anthony Tran\inst{1} \hspace{-2.0em} \And \hspace{-2.0em} J. D. Tygar\inst{1} \\[.1em]}

\aistatsaddress{ \inst{1}Computer Science Division, UC Berkeley \hspace{-2.0em} \And \hspace{-1.0em} \inst{2}Intel Labs Berkeley \hspace{-2.0em} \And \hspace{-2.0em} \inst{3}School of Computer Science, CMU }

]

\begin{abstract}
  Classifiers are often used to detect miscreant activities. We study
  how an adversary can efficiently query a classifier to elicit
  information that allows the adversary to evade detection at
  near-minimal cost. We generalize results of \citet{lowd05adversarial}
  to convex-inducing classifiers. We present algorithms that construct
  undetected instances of near-minimal cost using only polynomially
  many queries in the dimension of the space and without reverse
  engineering the decision boundary.
\end{abstract}

\section{INTRODUCTION}

Machine learning is often used
to filter or detect miscreant activities in a variety of applications;
\eg\ spam, intrusion, virus, and fraud detection. All known detection
techniques have blind spots; \ie\ classes of miscreant activity that fail
to be detected.
While learning allows the detection algorithm to adapt over time,
constraints on the learning algorithm also may allow an adversary to
programmatically find these vulnerabilities. We
consider how an adversary can systematically discover blind spots by
querying the learner to find a low cost instance that the detector does
not filter. Consider a spammer who wishes to minimally
modify a spam message so it is not classified as a spam.
By observing the
responses of the spam detector,
the spammer can search for a
modification while using few queries.

The problem of near optimal evasion (\ie\ finding a low cost negative
instance with few queries) was first posed by \citet{lowd05adversarial}. We continue this line of research by
generalizing it to the family of \convexClass---classifiers that
partition their instance space into two sets: one of which is convex.
\ConvexClass\ are a natural family to
examine as they include linear classifiers, 
anomaly detection classifiers using
bounded PCA~\citep{LCD04}, anomaly detection algorithms that use
hyper-sphere boundaries~\citep{PRML}, and other more complicated
bodies.

We also show that near-optimal evasion does not require reverse engineering
the classifier.
The algorithm of \citet{lowd05adversarial} for evading linear classifiers
reverse-engineers the decision boundary. Our algorithms for evading
 \convexClass\ do not require fully estimating the classifier's
boundary (which is hard in the general case; see \citealp{LearningConvexIsHard})
or reverse-engineering the classifier's state. Instead, we directly search for a minimal cost-evading instance.
Our algorithms require only polynomial-many queries, with one algorithm 
solving the linear case with fewer queries than the previously-published
reverse-engineering technique.


\stitle{Related Work.}
\label{sec:related-work}
\citet{dalvi04adversarial} uses a cost-sensitive game theoretic approach to
patch a classifier's blind spots. They
construct a modified classifier designed to detect optimally modified
instances. This work is complementary to our own; we examine optimal evasion
strategies while they have studied mechanisms for adapting the classifier.
In this work we assume the classifier is not adapting during evasion.

A number of authors have studied evading 
intrusion detector systems
(IDSs)~\citep{tan-etal-2002-undermining,wagner-soto-2002-mimicry}.
In exploring \emph{mimicry attacks} these authors demonstrated that
real IDSs could be fooled by modifying exploits to mimic normal
behaviors. These authors used offline analysis of the IDSs to
construct their modifications; by contrast, our modifications are optimized
by querying the classifier.

The field of active learning also studies
a form of query based optimization~\citep{ActLearning}.
While both active learning and near-optimal evasion explore optimal
querying strategies, the objectives for these two settings are quite
different (see Section~\ref{sec:opt-evasion}).

%

\section{PROBLEM SETUP}

\eat{
\begin{table}
  \begin{center}
  \begin{tabular}{|c|l|}
    \hline
    $\xspace$ &
      Space of data instances (continuous) \\
    $\yspace$ &
      Space of data labels; $\yspace = \set{\negLbl,\posLbl}$ \\
    $\classSpace$ &
      Space of classifiers \\
    \hline
    $\classifier \in \classSpace$ &
      Classifier function $\classifier : \xspace \mapsto \yspace$ \\
    $\xminus$ &
      Set of negative instances; $\xminus = \set[\prediction{\dpt}=\negLbl]{\dpt\in\xspace}$ \\
    $\xplus$ &
      Set of positive instances; $\xplus = \set[\prediction{\dpt}=\posLbl]{\dpt\in\xspace}$ \\
    \hline
    $\dpt \in \xspace$ &
      Data instance \\
    $\xtarget \in \xplus$ &
      Adversary's target instance \\
    $\dpt^- \in \xminus$ &
      Adversary's negative instance \\
    \hline
    $D$ &
      Dimension of $\xspace$ \\
    $\coordinVect{d}$ &
      Coordinate vector for feature $d$ \\
    \hline
  \end{tabular}
  \caption{Notation in this paper.}
  \label{tab:notation}
  \end{center}
\end{table}
}


We begin by introducing our notation and assumptions. First, we
assume that
instances are represented in $\dims$-dimensional Euclidean space
$\xspace=\mathbb{R}^\dims$. Each component of an
instance $\dpt \in \xspace$ is a \term{feature} which we denote as
$\dpt_d$. We denote each coordinate vector of the form
$(0,\ldots,1,\ldots,0)$ with a $1$ only at the \nth{d}{th} feature as
$\coordinVect{d}$. We assume that the feature space is known to the
adversary and any point in $\xspace$ can be queried.

We further assume the target classifier $\classifier$ belongs to
a family $\classSpace$. Any classifier
$\classifier \in \classSpace$ is a mapping from 
$\xspace$ to the labels \negLbl\ and \posLbl; \ie\ $\classifier: \xspace \mapsto \set{\negLbl,\posLbl}$. We assume the
adversary's attack will be against a fixed $\classifier$ so the
learning method and the training data used to select $\classifier$ are
irrelevant. We assume the adversary does not know $\classifier$ but does
know its family $\classSpace$.

We assume $\classifier \in \classSpace$ is deterministic and so
partitions $\xspace$ into a positive class $\xplus =
\set[\prediction{\dpt} = \posLbl]{\dpt \in \xspace}$ and a negative
class $\xminus = \set[\prediction{\dpt} = \negLbl]{\dpt
  \in \xspace}$. We take the negative set to be \emph{normal}
instances.
We assume the adversary is aware of at least one
instance in each class, $\dpt^- \in \xminus$ and $\xtarget \in
\xplus$, and can observe $\prediction{\dpt}$ for any
$\dpt$ by issuing a \term{membership query} (this last assumption
does not always hold in practice, see Section~\ref{sec:discuss} for
a more detailed discussion).

\subsection{Adversarial Cost}
\label{sec:adCost}

We assume the adversary has a notion of utility represented by a
cost function $\adCost : \xspace \mapsto \realnn$. The adversary
wishes to minimize $\adCost$ over the negative class, $\xminus$; \eg\ a
spammer wants to send spam that will be classified as normal email
(\negLbl) rather than as spam (\posLbl). We assume this cost function
is a distance to a positive target instance $\xtarget \in \xplus$ that is
most desirable to the adversary.
As with Lowd and Meek, we focus on the class of weighted \LP[1]
cost functions
\begin{equation}
  \adCostFunc{\dpt} = \sum_{d=1}^{\dims}{\ithCost{d} |\dpt_d - \xtarget_d|}\enspace,
  \label{eq:weightedL1}
\end{equation}
where $ 0 < \ithCost{d} < \infty$ is the cost the adversary
associates with the \nth{d}{th} feature. The $\LP[1]$-norm is a natural
measure of edit distance for email spam, while larger weights can
model tokens that are more costly to remove (\eg\ a payload URL).
We use $\ball{C}{\xtarget}$ to denote the ball centered at
$\xtarget$ with cost no more than $C$. We use $\oneball{C}{\dpt}$ to
refer specifically to a weighted \LP[1] ball.

\citet{lowd05adversarial} define \term{minimal adversarial cost (\MAC)} of a
classifier $\classifier$ to be the value
\[
  \function{\MAC}{\classifier,\adCost}
  \defAs \inf_{\dpt \in \xminus[\classifier]}\left[ \adCostFunc{\dpt}
  \right] \enspace.
\] 
They further define a data point to be an $\multGoal$-approximate
\emph{instance of minimal adversarial cost (\kIMAC)} if it is a
negative instance with cost no more than a factor $(1+\multGoal)$ of
the \MAC; \ie\ every \kIMAC\ is a member of the
set\footnote{We use the term \kIMAC\ to refer both to this set and
  members of it. The usage will be clear from the context.}
\begin{equation}
  \label{eq:kIMAC}
  \function{\kIMAC}{\classifier,\adCost} \defAs
  \set[\adCostFunc{\dpt} \le 
  (1+\multGoal)\cdot\function{\mathrm{MAC}}{\classifier,\adCost}]
  {\dpt \in \xminus[\classifier]}
\end{equation}
The adversary's goal is to find an \kIMAC\ instance
efficiently, while issuing as few queries as possible. 


\subsection{Search Terminology}
\label{sec:binSearch}

An \kIMAC\ instance is \term{multiplicatively optimal}; \ie\ it
is within a \emph{factor} of $(1+\multGoal)$ of the minimal
cost. We also consider \term{additive optimality};
\ie\ requiring a \kIMAC[\addGoal] to be no more than $\addGoal$
\emph{greater} than the minimal cost. The algorithms we present can
achieve either criterion given initial bounds $\minCost$ and $\maxCost$
such that $\minCost
\le \MAC \le \maxCost$.
If we can determine whether an intermediate cost establishes a new
upper or lower bound on \MAC, then binary search strategies can
iteratively reduce the \nth{t}{th} gap between $\maxCost_t$ and
$\minCost_t$. We now provide common terminology for the binary search
and in Section~\ref{sec:conv-evasion} we use convexity to establish a
new bound at each iteration.


In the \nth{t}{th} iteration of an additive binary search, $G_t^{(+)} =
\maxCost_t - \minCost_t$ is the additive gap between the \nth{t}{th} bounds.
The search uses a proposal step of
$C_t=\frac{\maxCost_t+\minCost_t}{2}$, a stopping
criterion of $G_t^{(+)} \le \addGoal$ and terminates in
\begin{equation}
  L^{(+)} = \left\lceil\log_2\left[(\maxCost-\minCost) / \addGoal\right]
  \right\rceil
  \label{eq:Ladd}
\end{equation}
steps. Binary search has the best worst-case query complexity for
achieving \addGoal-additive optimality.

Binary search can be adapted for multiplicative optimality:
by writing $\maxCost=2^a$ and $\minCost=2^b$, the
multiplicative condition becomes $a-b \le \log_2(1+\multGoal)$, an
additive optimality condition. Thus, binary search on the exponent
best achieves multiplicative optimality. The multiplicative gap of the
\nth{t}{th} iteration is $G_t^{(\ast)} = \maxCost_t / \minCost_t$. The
\nth{t}{th} query is $C_t = \sqrt{\maxCost_t\cdot\minCost_t}$, the
stopping criterion is $G_t^{(\ast)} \le 1+\multGoal$ and it stops
in
\begin{equation}
  L^{(\ast)} = \left\lceil\log_2\left[\log_2\left(\maxCost/\minCost\right) / \log_2(1+\multGoal)\right]\right\rceil
  \label{eq:Lmult}
\end{equation}
steps. Multiplicative optimality only makes sense when
both \maxCost\ and \minCost\ are strictly positive. 


For this paper, we only address multiplicative optimality and define
$L = L^{(\ast)}$ and $G_t = G_t^{(\ast)}$, but note that our techniques
also apply to additive optimality.

\subsection{Near-Optimal Evasion}
\label{sec:opt-evasion}

\citet{lowd05adversarial} introduced the problem of \term{adversarial classifier
  reverse engineering (ACRE)} where a family of classifiers is called
\term{\kACRE} if there is an efficient query-based algorithm for
finding an \kIMAC.
In generalizing their result, we slightly alter their definition
of query complexity. First, to quantify query complexity we only use the dimension $\dims$ and the
number of steps $L$ required by a univariate binary search.
Second, we assume the adversary only has two initial points $\dpt^-
\in \xminus$ and $\xtarget \in \xplus$ (the original setting
required a third $\dpt^+ \in \xplus$). Finally, our algorithms do not
reverse engineer so ACRE would be a misnomer.  Instead we call the
overall problem \term{Near-Optimal Evasion} and replace \kACRE\ with
\begin{quote}
  A family of classifiers \classSpace\ is \term{\kSearchable} under a
  family of cost functions \adCostSpace\ if for all $\classifier \in
  \classSpace$ and $\adCost \in \adCostSpace$, there is an algorithm
  that finds $\dpt \in \function{\kIMAC}{\classifier,\adCost}$ using
  polynomially many membership queries in $\dims$ and $L$.
\end{quote}

Reverse engineering is an expensive approach for near-optimal
evasion in the general case. Efficient query-based reverse engineering for
$\classifier \in \classSpace$ is sufficient for
minimizing $\adCost$ over the estimated negative space.
However, the requirements for finding an \kIMAC\ differ 
from the objectives of reverse engineering approaches such as active
learning. Both use queries to reduce the size of version space
$\hat{\classSpace} \subset \classSpace$. However reverse engineering
minimizes the expected number of disagreements between members of
$\hat{\classSpace}$.
In contrast, to find an \kIMAC, we only need to provide a single instance
$\dpt^\dagger \in \function{\kIMAC}{\classifier,\adCost}$
for all $\classifier \in \hat{\classSpace}$, while
leaving the classifier largely unspecified.
We present algorithms for \kIMAC\ search on a family of
classifiers that generally cannot be efficiently reverse
engineered---the queries we construct necessarily elicit an \kIMAC\ only.
\section{EVASION OF CONVEX CLASSES}
\label{sec:conv-evasion}


We generalize \kSearchability\ to the family of \term{\convexClass}
$\classSpace^\mathrm{convex}$ that partition feature space $\xspace$
into a positive and negative class, one of which is convex. The
\convexClass\ include linear classifiers, 
one-class classifiers that
predict anomalies by thresholding the log-likelihood of a log-concave
(or uni-modal) density function, and quadratic classifiers of the form
$\dpt^\top \mat{A} \dpt + \vec{b}^\top \dpt + c \ge 0$ if $\mat{A}$ is
semidefinite.  The \convexClass\ also include complicated families
such as the set of all intersections of a countable number of
halfspaces, cones, or balls.

We construct efficient algorithms for query-based optimization of the
(weighted) \LP[1] cost of Eq.~\eqref{eq:weightedL1} for 
\convexClass. There appears to be an asymmetry depending on
whether the positive or negative class is convex. When the positive
set is convex, determining whether $\oneball{C}{\xtarget} \subset
\xplus$ only requires querying the vertices of the ball. When the
negative set is convex, determining whether
$\oneball{C}{\xtarget} \cap \xminus = \emptyset$ is difficult since
the intersection need not occur at a vertex. We present an
efficient algorithm for optimizing an \LP[1] cost when
$\xplus$ is convex and a polynomial random algorithm for optimizing
any convex cost when $\xminus$ is convex.

The algorithms we present achieve multiplicative optimality via binary
search; we use $L$ as the number of phases required by binary search,
$\maxCost = \adCostFunc{\dpt^-}$ as an initial upper bound on the
\MAC\ and assume there is some $\minCost > 0$ that lower bounds the
\MAC\ (\ie\ $\xtarget$ is in the interior of $\xplus$). This condition
eliminates the degenerate case for which $\xtarget$ is on the boundary
of $\xplus$ where $\function{\MAC}{\classifier,\adCost}=0$ and
$\function{\kIMAC}{\classifier,\adCost}=\emptyset$.


\subsection{\kIMAC\ Search for a Convex $\xplus$}
\label{sec:pos-convex}


Solving the \kIMAC\ search problem when $\xplus$ is convex is hard in
the general case of convex cost $\adCostFunc{\cdot}$.  We
demonstrate algorithms for the (weighted) \LP[1] cost that solve the
problem as a binary search.
Namely, given initial costs $\minCost$ and $\maxCost$ that bound the
\MAC, our algorithm can efficiently determine whether
$\oneball{C}{\xtarget} \subset \xplus$ for any intermediate cost
$\minCost < C < \maxCost$. If the \LP[1] ball is contained in
$\xplus$, then $C$ becomes the new lower bound $\minCost$. Otherwise
$C$ becomes the new upper bound $\maxCost$.  Since our objective
Eq.~\eqref{eq:kIMAC} is to obtain multiplicative optimality, our
steps will be $C_t=\sqrt{\minCost_{t-1}\cdot\maxCost_{t-1}}$ (see
Section~\ref{sec:binSearch}). We now explain how we exploit the
properties of the (weighted) \LP[1] ball and convexity of $\xplus$ to
efficiently determine whether $\oneball{C}{\xtarget} \subset \xplus$.

The existence of an efficient query algorithm relies on three facts:
(1) $\xtarget \in \xplus$; (2) every weighted \LP[1]
cost $C$-ball centered at $\xtarget$ intersects $\xminus$ only if at least
one of its vertices is in $\xminus$; and (3) $C$-balls only have $2\cdot
\dims$ vertices. We formalize the second fact as follows.

\begin{lemma}
  \label{thm:optAxis}
  For all $C > 0$, if there exists some $\dpt \in \xminus$ that
  achieves a cost of $C=\adCostFunc{\dpt}$, then there is some
  feature $d$ such that a vertex of the form
  \begin{equation}
    \xtarget \pm \tfrac{C}{\ithCost{d}} \coordinVect{d}
    \label{eq:optAxisVect}
  \end{equation}
  is in $\xminus$ (and also achieves cost $C$ by
  Eq.~\ref{eq:weightedL1}).
\end{lemma}
\begin{proof}
  Suppose not; then there is some $\dpt \in \xminus$ such that
  $\adCostFunc{\dpt}=C$ and $\dpt$ has $M\ge 2$ features that differ
  from $\xtarget$. Let $\set{d_1,\ldots,d_M}$ be the differing
  features and let $b_{d_i} = \sign\left(\dpt_{d_i} -
    \xtarget_{d_i}\right)$ be the sign of the difference between
  $\dpt$ and $\xtarget$ along the $d_i$-th feature. Let $\vec{e}_{d_i}
  = \xtarget + \tfrac{C}{\ithCost{d_i}} \cdot b_{d_i} \cdot
  \coordinVect{d_i}$ be a vertex of the form of
  Eq.~\eqref{eq:optAxisVect} which has cost $C$ (from
  Eq.~\ref{eq:weightedL1}).  The $M$ vertices $\vec{e}_{d_i}$ form a
  simplex of cost $C$ on which $\dpt$ lies. If all $\vec{e_{d_i}} \in
  \xplus$, then the convexity of $\xplus$ implies that $\dpt \in
  \xplus$ which violates our premise.  Thus, if any instance in
  $\xminus$ achieves cost $C$, there is always a vertex of the form
  Eq.~\eqref{eq:optAxisVect} in $\xminus$ that also achieves cost $C$.
\end{proof}

As a consequence, if all vertices of any $C$ ball
$\oneball{C}{\xtarget}$ are positive, then all $\dpt$ with
$\adCostFunc{\dpt} \le C$ are positive thus establishing $C$ as a
lower bound on the \MAC. Conversely, if any of the vertices of
$\oneball{C}{\xtarget}$ are negative, then $C$ is an upper
bound. Thus, by querying all $2\cdot \dims$ vertices of
$\oneball{C}{\xtarget}$, we either establish $C$ as a new lower or
upper bound on the \MAC. By performing a binary search on $C$ we
iteratively halve the multiplicative gap between our bounds until it
is within a factor of $1+\multGoal$. This yields an \kIMAC\ of the form
of Eq.~\eqref{eq:optAxisVect}.

A general form of this multiline search procedure is presented as
Algorithm~\ref{alg:mls} which simultaneously searches along all
unit-cost directions in the set $\aset{W}$. At each step,
\textsc{MultiLineSearch} issues at most $|\aset{W}|$ queries to
determine whether $\oneball{C}{\xtarget} \subset \xplus$. Once a
negative instance is found at cost $C$, we cease further queries at
cost $C$ since a single negative instance is sufficient to establish a
lower bound. We call this policy \term{lazy querying}. Further, when
an upper bound is established for a cost $C$, our algorithm also
prunes all directions that were positive at cost $C$. This pruning is
sound; by the convexity assumption we know that the pruned direction
is positive for all costs
less than our new upper bound $C$.  Applying \textsc{MultiLineSearch}
to the $2\cdot\dims$ axis-aligned directions yields an \kIMAC\ for any
(weighted) \LP[1] cost with no more than $2\cdot\dims L$ queries but
at least $\dims+L$ queries.  Thus the algorithm is $\bigO{\dims L}$.

\ban{Somewhere we should note that MLS does not rely on its directions being vertices of the \LP[1] ball although those vertices are sufficient to span the \LP[1] ball. MLS is agnostic to where it's search directions point but does rely heavily on convexity and the fact that all it's search directions are of unit cost (and can be increased multiplicatively).}

\ban{It'd be good to extend to cases $\ithCost{d} = 0$ and $\ithCost{d} = \infty$.}

\ban{We need common syntax for algorithm names\ldots maybe small-caps.}

\begin{figure}
\begin{minipage}{\linewidth}
  \begin{Algorithm}{Multi-line Search}{alg:mls}
    \STATE $\function{MLS}{\aset{W},\xtarget,\dpt^-,C^+,C^-,\multGoal}$
    \STATE $\dpt^{\ast} \gets \dpt^{-}$
    \WHILE{$C^- / C^+ > 1+\multGoal$}
      \STATE $C \gets \sqrt{C^+ \cdot C^-}$
      \FORALL{$\vec{e} \in \aset{W}$}
        \STATE Query classifier: $f_\vec{e}^C \gets \prediction{\xtarget + C \vec{e}}$
        \IF{$f_\vec{e}^C = \negLbl$}
          \STATE $\dpt^{\ast} \gets \xtarget + C \vec{e}$
          \STATE Prune $\vec{i}$ from $\aset{W}$ if $f_\vec{i}^C = \posLbl$
          \STATE \textbf{break for-loop}
        \ENDIF
      \ENDFOR
      \STATE \textbf{if} $\forall \vec{e} \in \aset{W} \; f_\vec{e}^C = \posLbl$ \textbf{then} $C^+ \gets C$
      \STATE \textbf{else} $C^- \gets C$
    \ENDWHILE
    \STATE \textbf{return:} $\dpt^{\ast}$
  \end{Algorithm}
\end{minipage}
\end{figure}

\subsubsection{$K$-step Multi-Line Search}
\label{sec:kmls}

The \textsc{MultiLineSearch} algorithm is $2\cdot\dims$ simultaneous
binary searches (breadth-first). Instead we could search sequentially
(depth-first) and obtain a best case of $\bigO{\dims+L}$ and
worst case of $\bigO{\dims \cdot L}$ but for exactly the opposite
convex bodies.
We therefore propose an algorithm that mixes these strategies.  At
each phase, the \KMLS\ (Algorithm~\ref{alg:kmls}) chooses a single
direction $\vec{e}$ and queries it for $K$ steps to generate candidate
bounds $B^-$ and $B^+$ on the \MAC. The algorithm makes substantial
progress without querying other directions. It then iteratively
queries all remaining directions at the candidate lower bound $B^+$.
Again we use lazy querying and stop as soon as a negative instance is
found. We show that for $K=\lceil \sqrt{L} \rceil$, the algorithm
achieves a delicate balance between breadth-first and depth-first
approaches to attain a better worst-case complexity.

\begin{figure}
\begin{center}
\begin{minipage}{\linewidth}
  \begin{Algorithm}{$K$-Step Multi-line Search}{alg:kmls}
    \STATE $\function{KMLS}{\aset{W},\xtarget,\dpt^-,C^+,C^-,\multGoal,K}$
    \STATE $\dpt^{\ast} \gets \dpt^{-}$
    \WHILE{$C^- / C^+ > 1+\multGoal$}
      \STATE Choose a direction $\vec{e} \in \aset{W}$
      \STATE $B^+ \gets C^+$
      \STATE $B^- \gets C^-$
      \FOR{$K$ steps}
        \STATE $B \gets \sqrt{B^+ \cdot B^-}$
        \STATE Query classifier: $f_\vec{e} \gets \prediction{\xtarget + B \vec{e}}$
        \STATE \textbf{if} $f_\vec{e} = \posLbl$ \textbf{then} $B^+ \gets B$
        \STATE \textbf{else} $B^- \gets B$ \textbf{and} $\dpt^{\ast} \gets \xtarget + B \vec{e}$
      \ENDFOR

      \FORALL{$\vec{i} \neq \vec{e} \in \aset{W}$}
        \STATE Query classifier: $f_\vec{i} \gets \prediction{\xtarget + (B^+) \vec{i}}$
        \IF{$f_\vec{i} = \negLbl$}
          \STATE $\dpt^{\ast} \gets \xtarget + (B^+) \vec{i}$
          \STATE Prune \vec{k} from $\aset{W}$ if $f_\vec{k} = \posLbl$
          \STATE \textbf{break for-loop}
        \ENDIF
      \ENDFOR

      \STATE $C^- \gets B^-$
      \STATE \textbf{if} $\forall \vec{i} \in \aset{W} \; f_\vec{i} = \posLbl$ \textbf{then} $C^+ \gets B^+$
      \STATE \textbf{else} $C^- \gets B^+$
    \ENDWHILE
    \STATE \textbf{return:} $\dpt^{\ast}$
  \end{Algorithm}
\end{minipage}
\end{center}
\end{figure}

To analyze the worst case of \KMLS, we consider
a \textit{defender} that maximizes the number of
queries. We refer to the querier as the \textit{adversary}.

\begin{theorem}
  \label{thm:sqrtL}
  Algorithm~\ref{alg:kmls} will find an \kIMAC\ with at most
  $\bigO{L + \sqrt{L} |\aset{W}|}$ queries for $K = \lceil \sqrt{L} \rceil$.
\end{theorem}

\begin{proof}
  During the $K$ steps of binary search, regardless of how the
  defender responds, the candidate gap along $\vec{e}$
  will shrink by an exponent of $2^{-K}$; \ie
  \begin{eqnarray}
    B^- / B^+ = \left(\maxCost / \minCost\right)^{2^{-K}}\enspace. \label{eq:gap-shrink}
  \end{eqnarray}
  The primary decision for the defender occurs when the
  adversary begins querying other directions than $\vec{e}$. At
  iteration $t$, it has 2 options:
  \begin{verse}
  Case 1 ($t \in \aset{C}_1$): Respond with \posLbl\ for all
    remaining directions. Here the bounds $B^+$ and $B^-$ are verified
    and thus the gap is reduced by an exponent of $2^{-K}$.
  \end{verse}
  \begin{verse}
  Case 2 ($t \in \aset{C}_2$): Choose at least 1 direction to respond
    with \negLbl. Here the defender can make the gap
    decrease negligible but also must choose some number $E_t \ge 1$ of
    eliminated directions.
  \end{verse}
  By conservatively assuming the gap only decreases in case 1, the
  total number of queries is bounded regardless of the order in which
  the cases are applied. Thus if $t \in \aset{C}_1$ we
  have $G_t = G_{t-1}^{2^{-K}}$; otherwise we have $G_t =
  G_{t-1}$. Thus
  \begin{equation}
    \label{eq:gap-tradeoff}
    |\aset{C}_1| \le \left\lceil\tfrac{L}{K}\right\rceil\enspace,
  \end{equation}
  since we need a total of $L$ binary search steps and each case 1
  iteration does $K$ of them.

  Every case 1 iteration makes exactly $K + |\aset{W}_t| - 1$ queries. The
  size of $\aset{W}_t$ is controlled by the defender, but
  we can bound it by $|\aset{W}|$. This and
  Eq.~\eqref{eq:gap-tradeoff} bound the number of queries used in case
  1 ($Q_1$) by \eat{
  \begin{eqnarray*}
    Q_1 & = & \sum_{t \in C_1}{(K + |\aset{W}_t| - 1)} \\
        & \le & \left\lceil\frac{L}{K}\right\rceil \cdot K + \left\lceil\frac{L}{K}\right\rceil \cdot (|\aset{W}|-1) \\
        & \le & L + K + \left\lceil\frac{L}{K}\right\rceil \cdot \left(|\aset{W}|-1\right)
  \end{eqnarray*}
}
  \begin{eqnarray*}
    Q_1 = \sum_{t \in C_1}{(K + |\aset{W}_t| - 1)}
        \le L + K + \left\lceil\tfrac{L}{K}\right\rceil \cdot \left(|\aset{W}|-1\right) 
  \end{eqnarray*}
  Each case 2 iteration uses exactly $K + E_t$ queries and
  eliminates $E_t \ge 1$ directions. 
  Since a case 2 iteration
  eliminates at least 1 direction, $|\aset{C}_2| \le
  |\aset{W}| - 1$ and moreover, $\sum_{t \in \aset{C}_2}{E_t} \le |\aset{W}| - 1$ since each
  direction can only be eliminated once. Thus
\eat{
  \begin{eqnarray*}
    Q_2 & = & \sum_{i \in \aset{C}_2}{(K + E_t)} \\
        & \le & |\aset{C}_2| \cdot K + |\aset{W}|-1 \\
        & \le & \left(|\aset{W}|-1\right)\left(K + 1\right)
        \enspace.
  \end{eqnarray*}
}
  \begin{eqnarray*}
    Q_2 = \sum_{i \in \aset{C}_2}{(K + E_t)}
        \le \left(|\aset{W}|-1\right)\left(K + 1\right)
        \enspace,
  \end{eqnarray*}
  and so the total queries used by Algorithm~\ref{alg:kmls} is
\eat{
  \begin{eqnarray*}
    Q = Q_1+Q_2 & < & L + \left\lceil\frac{L}{K}\right\rceil \cdot |\aset{W}| + K\cdot|\aset{W}| + |\aset{W}| \\
      & = & L + \left(\left\lceil\frac{L}{K}\right\rceil + K + 1\right) |\aset{W}|
  \end{eqnarray*}
}
  \begin{eqnarray*}
    Q = Q_1+Q_2 < L + \left(\left\lceil\tfrac{L}{K}\right\rceil + K + 1\right) |\aset{W}| \enspace,
  \end{eqnarray*}
  which is minimized by $K=\lceil \sqrt{L} \rceil$.
  Substituting this for $K$ and using
  $L/\lceil\sqrt{L}\rceil \le \sqrt{L}$ we have
  \[
    Q < L + (2 \lceil\sqrt{L}\rceil + 1) |\aset{W}| \enspace. \qedhere
  \]
\end{proof}

As a consequence of Theorem~\ref{thm:sqrtL}, finding an \kIMAC\ with
Algorithm~\ref{alg:kmls} for a (weighted) \LP[1] cost requires
$\bigO{L + \sqrt{L}\dims}$ queries. 
Moreover, linear classifiers are a special case of
\convexClass\ for our \KMLS\ algorithm. 
Thus \KMLS\ improves on the reverse-engineering technique's
$\bigO{L\dims}$ queries and applies to a broader family.

\subsubsection{Lower Bound}

Here we find lower bounds on the number of queries required by any
algorithm to find an \kIMAC\ when $\xplus$ is convex. Notably,
since an \kIMAC\ uses multiplicative optimality, we
incorporate a lower bound $r > 0$ on the \MAC\ into our statement.



\begin{theorem}
  \label{thm:mlower}
  Consider any $\dims>0$, $\xtarget\in \reals^\dims$, $\dpt^- \in \reals^\dims$,
  $0 < r < R= \adCostFunc{\dpt^-}$ and $\multGoal \in \left(0,\frac{R}{r}-1\right)$. For
  all query algorithms submitting $N<\max\{\dims,L^{(\ast)}\}$ queries, there exist two classifiers inducing convex positive
  classes in $\reals^D$ such that 
  \begin{enumerate}
    \item Both positive classes properly contain $\ball{r}{\xtarget}$;
    \item Neither positive class contains $\dpt^-$;
    \item The classifiers return the same responses on the algorithm's $N$ queries; and
    \item The classifiers have no common \kIMAC[\multGoal].
  \end{enumerate}
  That is, in the worst-case all query algorithms for convex positive classes
  must submit at least $\max\{\dims,L^{(\ast)}\}$ membership queries in order to
  be multiplicative \multGoal-optimal.
\end{theorem}

\begin{proof}
  Suppose some query-based algorithm submits $N$ membership queries
  $\dpt^{1},\ldots,\dpt^{N}$ to the classifier. For the algorithm to
  be \multGoal-optimal, these queries must constrain all consistent
  positive convex sets to have a common point among
  their \kIMAC[\multGoal] sets.

  First we consider the case that $N\geq L$.
  Then by assumption $N<\dims$.
  Suppose classifier $\classifier$ responds as
  \[
  \prediction{\dpt} = \begin{cases}
    +1\:, & \mbox{if } \adCostFunc{\dpt} < R \\
    -1\:, & \mbox{otherwise}
  \end{cases}\enspace.
  \]
  For this classifier, $\xplus$ is convex, $\ball{r}{\xtarget} \subset
  \xplus$, and $\dpt^- \notin \xplus$. Moreover, since $\xplus$ is the
  open ball of cost $R$, $\function{\MAC}{\classifier,\adCost}=R$.

  Consider an alternative classifier $g$ that responds
  identically to $\classifier$ for $\dpt^{1},\ldots,\dpt^{N}$ but has
  a different convex positive set $\xplus[g]$. Without loss
  of generality, suppose the first $M \le N$ queries are positive
  and the remaining are negative. Let $\aset{G} =
  \function{conv}{\dpt^1,\ldots,\dpt^M}$; that is, the convex hull of
  the $M$ positive queries. Now let $\xplus[g]$ be the
  convex hull of the union of $\aset{G}$ and the $r$-ball around $\xtarget$:
  $\xplus[g] =
  \function{conv}{\aset{G}\cup\ball{r}{\xtarget}}$. Since $\aset{G}$
  contains all positive queries and $r < R$, the convex set
  $\xplus[g]$ is consistent with the responses from
  $\classifier$, $\ball{r}{\xtarget} \subset \xplus$, and $\dpt^-
  \notin \xplus$. Further, since $M \le N < \dims$, $\aset{G}$ is
  contained in a proper subspace of $\reals^\dims$ whereas
  $\ball{r}{\xtarget}$ is not. Hence,
  $\function{\MAC}{g,\adCost} = r$. Since the accuracy \multGoal\ is less than $\frac{R}{r} - 1$, any \kIMAC\ of $g$ must have cost less than $R$ whereas any \kIMAC\ of $\classifier$ must have cost greater than or equal to $R$.
  Thus we have constructed two \convexClass\
  $\classifier$ and $g$ with consistent query
  responses but with no common \kIMAC.

  Second, we consider the case that $N < L$. First, recall our
  definitions: $\maxCost_0 = R$ is the initial upper bound on the
  \MAC, $\minCost_0 = r$ is the initial lower bound on the \MAC, and
  $G_t^{(\ast)} = \maxCost_t / \minCost_t$ is the gap between the upper
  bound and lower bound at iteration $t$. Here the defender
  $\classifier$ responds with
  \begin{equation*}
    \prediction{\dpt^t} = \begin{cases}
      +1\:, & \mbox{if }\adCostFunc{\dpt^t} \le \sqrt{\maxCost_{t-1}\cdot\minCost_{t-1}}\\
      -1\:, & \mbox{otherwise}\end{cases}\enspace.
  \end{equation*}
  This strategy ensures that at each iteration $G_t \ge \sqrt{G_{t-1}}$ and
  since the algorithm can not
  terminate until $G_N \le 1+\multGoal$,
  we have $N \ge L^{(\ast)}$ from Eq.~\eqref{eq:Lmult}.  As in the
  $N\geq L$ case we have constructed two \convexClass\ with consistent
  query responses but with no common \kIMAC. The first classifier's
  positive set is the smallest cost-ball enclosing all positive
  queries, while the second classifier's positive set is the largest
  cost-ball enclosing all positive queries but no negatives. The \MAC\
  values of these sets differ by more than a factor of $(1+\multGoal)$
  if $N < L^{(\ast)}$ so they have no common \kIMAC.
\end{proof}

This theorem shows that \multGoal-multiplicative optimality requires
$\bigOmega{\dims + L}$ queries.
Hence 
\KMLS\ (Algorithm~\ref{alg:kmls}) has close to the optimal query complexity.

\subsection{\kIMAC\ Learning for a Convex $\xminus$}

\ban{Argue somewhere that subspaces and non-full rank convex sets are
  inherently hard.}

In this section we consider minimizing a convex cost function
$\adCost$ (we focus on weighted \LP[1] costs in
Eq.~\ref{eq:weightedL1}) when the feasible set $\xminus$ is convex.
Any convex function can be efficiently minimized within a
known convex set \eg\ using the Ellipsoid or Interior Point
methods~\citep{ConvexOptimization}. However, in our problem the convex set is
only accessible through queries.  We use a randomized polynomial
algorithm of \citet{bertsimas04convexrw} to
minimize the cost given an initial $\dpt^-
\in \xminus$. For any fixed cost $C^t$ we use their algorithm to
determine (with high probability) whether $\xminus$ intersects with
$\ball{C^t}{\dpt^A}$; \ie\ whether or not $C^t$ is a new lower or
upper bound on the \MAC.  With high probability,
we find an \kIMAC\ in no more than $L$ repetitions using binary search.

\begin{figure*}[t]
\begin{minipage}{0.55\linewidth}
\begin{Algorithm}{Intersect Search}{alg:intersect}
  \STATE $\function{IntersectSearch} {\aset{P}^0, \aset{Q} = \set{\dpt^j \in \aset{P}^0}, C}$
  \FORALL{$s = 1 \ldots T$}
  \STATE (1) Generate $2N$ samples $\set{\dpt^j}_{j=1}^{2N}$
  \STATE \quad Choose $\dpt$ from $\aset{Q}$
  \STATE \quad $\dpt^j \gets \function{HitRun}{\aset{P}^{s-1}, \aset{Q}, \dpt^j}$
  \STATE (2) If any $\dpt^j$, $\adCostFunc{\dpt^j} \le C$ terminate the for-loop
  \STATE (3) Put samples into 2 sets of size $N$ 
  \STATE \quad $\aset{R} \gets \set{\dpt^j}_{j=1}^{N}$ and $\aset{S} \gets \set{\dpt^j}_{j=2N+1}^{2N}$
  \STATE (4) $\vec{z}^{s} \gets \frac{1}{N}\sum_{\dpt^j\in\aset{R}}{\dpt^j}$
  \STATE (5) Compute $\aset{H}_{\vec{z}^{s}}$ using Eq.~\eqref{eq:halfspace}
  \STATE (6) $\aset{P}^{s} \gets \aset{P}^{s-1} \cap \aset{H}_{\vec{z}^{s}}$
  \STATE (7) Keep samples in $\aset{P}^s$
  \STATE \quad $\aset{Q} \gets \set{\dpt\in \aset{S} \land \dpt \in \aset{P}^s}$
  \ENDFOR
  \STATE \textbf{Return:} the found $[\dpt_j, \aset{P}^s, \aset{Q}]$; or No Intersect 
\end{Algorithm}
\end{minipage}\quad
\begin{minipage}{0.42\linewidth}
\begin{Algorithm}{Hit-and-Run Sampling}{alg:hitandrun}
  \STATE $\function{HitRun} {\aset{P}, \set{\vec{y}^j}, \dpt^0}$
    \FORALL{$i = 1 \ldots K$}
    \STATE Pick a random direction: 
    \STATE \quad $\nu_j \sim \gauss{0}{1}$
    \STATE \quad $\vec{v} \gets \sum_j{\nu_j \vec{y}^j}$
    \STATE Find $\omega_1$ and $\omega_2$ s.t.
    \STATE \quad $\dpt^{i-1} - \omega_1 \vec{v} \notin \aset{P}$ and $\dpt^{i-1} + \omega_2 \vec{v} \notin \aset{P}$
    \REPEAT
      \STATE $\omega \sim \function{Unif}{-\omega_1,\omega_2}$
      \STATE $\dpt^i \gets \dpt^{i-1} + \omega \vec{v}$
      \STATE \textbf{if} $\omega < 0$ \textbf{then} $\omega_1 \gets -\omega$ 
      \STATE \textbf{else} $\omega_2 \gets \omega$
    \UNTIL{$\dpt^i \in \aset{P}$}
    \ENDFOR
    \STATE \textbf{Return:} $\dpt^K$
  \end{Algorithm}
  \end{minipage}
\end{figure*}

\subsubsection{Intersection of Convex Sets}

We now outline Bertsimas and Vempala's query-based algorithm
for determining whether two convex sets 
intersect using a randomized Ellipsoid method. In particular $\aset{P}$
is only
accessible through membership queries and $\aset{B}$ provides a
separating hyperplane for any point outside it. 
They use efficient query-based approaches to uniformly sample from
$\aset{P}$ to produce sufficiently many samples such that cutting
$\aset{P}$ through the centroid of these samples with a separating
hyperplane from $\aset{B}$ significantly reduces the volume of
$\aset{P}$ with high probability.  Their algorithm thus constructs a
sequence of progressively smaller feasible sets $\aset{P}^s \subset
\aset{P}^{s-1}$ until either the algorithm finds a point in $\aset{P} \cap
\aset{Q}$ or it is highly unlikely that the sets intersect.

Our problem reduces to finding the intersection between $\xminus$ and
$\oneball{C^t}{\xtarget}$. Though $\xminus$ may be unbounded, we can
instead use $\aset{P}^0 = \xminus \cap
\oneball{2R}{\dpt^-}$ (where $R = 2\adCostFunc{\dpt^-}$) is a
subset of $\xminus$ that envelops all of $\oneball{C^t}{\xtarget}$
since $C^t < \adCostFunc{\dpt^-}$. We also assume there is some
$r>0$ such that an $r$-ball centered at $\dpt^-$ is contained in
$\xminus$.
We now detail this \textsc{IntersectSearch} procedure
(Algorithm~\ref{alg:intersect}).


The backbone of the algorithm is uniform sampling
from a bounded convex body by means of the \textsc{hit-and-run} random
walk technique introduced by \citet{smith-96-hnr} (Algorithm~\ref{alg:hitandrun}). Given an instance $\dpt^j \in
\aset{P}^{s-1}$, \textsc{hit-and-run} selects a random direction
$\vec{v}$ through $\dpt^j$ (we return to the selection of $\vec{v}$ in
Section~\ref{sec:sampling}). Since $\aset{P}^{s-1}$ is a bounded
convex set, the set $\Omega = \set[\dpt^j + \omega \vec{v} \in
  \aset{P}^{s-1}]{\omega}$ is a bounded interval representing all
points in $\aset{P}^{s-1}$ along direction $\vec{v}$.  Sampling
$\omega$ uniformly from $\Omega$ 
yields the next step of the walk; $\dpt^j + \omega \vec{v}$.  Under
the appropriate conditions (see Section~\ref{sec:sampling}),
\textsc{hit-and-run} generates a sample uniformly from the convex body
after $\bigOstar{\dims^3}$ steps\footnote{$\bigOstar{\cdot}$ denotes
  $\bigO{\cdot}$ without logarithmic terms.}~\citep{lovasz04hnrfast}.
%
%

Using \textsc{hit-and-run} we obtain $2N$ samples $\set{\dpt^j}$ from
$\aset{P}^{s-1}$ and check if any satisfy $\adCostFunc{\dpt^j} \le
C^t$. If so, $\dpt^j$ is in the intersection of $\xminus$ and
$\oneball{C^t}{\xtarget}$. Otherwise, we want to significantly reduce
the size of $\aset{P}^{s-1}$ without excluding any of
$\oneball{C^t}{\xtarget}$ so that sampling concentrates towards the
intersection (if it exists)---for this we need a separating hyperplane
of $\oneball{C^t}{\xtarget}$.  For any $\vec{y} \notin
\oneball{C^t}{\xtarget}$, the (sub)gradient of the weighted \LP[1]
cost given by
\begin{equation}
\label{eq:gradient}
\vec{h}_f^\vec{y} 
= \ithCost{f} \sign\left(\vec{y}_f - \xtarget_f\right)
\end{equation}
separates $\vec{y}$ and $\oneball{C^t}{\xtarget}$.

To achieve efficiency, we choose a point
$\vec{z} \in \aset{P}^{s-1}$ so that cutting $\aset{P}^{s-1}$ through
$\vec{z}$ with the hyperplane $\vec{h}^\vec{z}$ eliminates a
significant fraction of $\aset{P}^{s-1}$. To do so, $\vec{z}$ must be
centrally located within $\aset{P}^{s-1}$. We use the
empirical centroid of half of the samples $\vec{z} =
\frac{1}{N}\sum_{\dpt \in \aset{R}}{\dpt}$ (the other half will be
used in Section~\ref{sec:sampling}). We cut $\aset{P}^{s-1}$ with the
hyperplane $\vec{h}^\vec{z}$ through $\vec{z}$; \ie\ $\aset{P}^s =
\aset{P}^{s-1} \cap \aset{H}_{\vec{z}}$ where $\aset{H}_{\vec{z}}$ is
the halfspace
\begin{equation}
\label{eq:halfspace}
\aset{H}_\vec{z} =
\set[\dpt^\top\vec{h}^\vec{z} \le \vec{z}^\top\vec{h}^\vec{z}]{\dpt}
\enspace .
\end{equation}
As shown by Bertsimas and Vempala, this cut achieves
$\function{vol}{\aset{P}^{s}} \le \frac{2}{3}
\function{vol}{\aset{P}^{s-1}}$ with high probability if $N =
\bigOstar{\dims}$ and $\aset{P}^{s-1}$ is near-isotropic (see
Section~\ref{sec:sampling}).
Since the ratio of
volumes between the initial circumscribing and inscribing balls of the
feasible set is $\left(\frac{R}{r}\right)^\dims$, the algorithm can
terminate after $T = \bigO{\dims \log \frac{R}{r}}$ unsuccessful iterations
with a high probability that the intersection is empty.

Because every iteration in Algorithm~\ref{alg:intersect} requires
$N=\bigOstar{D}$ samples, each of which need $K=\bigOstar{D^3}$ random
walk steps, and there are $\bigOstar{D}$ iterations,
Algorithm~\ref{alg:intersect} requires $\bigOstar{D^5}$ queries.


\subsubsection{Sampling from a Convex Body}
\label{sec:sampling}

\eat{
\ban{What is the point of this???}
We use random samples from a convex body for two purposes: estimating
the centroid of the convex body (generally a hard
problem~\citep{rademacher-07-centerhard}) and maintaining the
conditions required for the hit-and-run sampler to continue to
generate points uniformly on a sequence of shrinking convex bodies.
}

Until this point, we assumed the \textsc{hit-and-run} random walk
efficiently produces uniformly random samples from any bounded convex
body $\aset{P}$ accessible through membership queries.  However, if
the body is severely elongated, randomly selected directions will
rarely align with the long axis of the body and our random walk will
take small steps (relative to the long axis) and mix slowly.
For the sampler to mix effectively, we need the convex body $\aset{P}$
to be \term{near-isotropic}; \ie\ for any unit vector $\vec{v}$,
$\expect{\dpt \sim \aset{P}}{\left(\vec{v}^\top\left(\dpt -
      \expect{\dpt \sim \aset{P}}{\dpt}\right)\right)^2}$ is bounded
between $1/2$ and $3/2$ of $\function{vol}{\aset{P}}$.

If the body is not near-isotropic, we can rescale $\xspace$ with an
appropriate affine transformation $\mat{T}$.
With sufficiently many samples from $\aset{P}$ we can estimate
$\mat{T}$ as their empirical covariance matrix. Instead, we rescale
$\xspace$ implicitly using a technique described by \citet{bertsimas04convexrw}.
We maintain a set $\aset{Q}$ of
sufficiently many uniform samples from the body $\aset{P}^s$ and in
\textsc{hit-and-run} we sample directions based on this set.  Because
the samples are distributed uniformly in $\aset{P}^s$, the directions
we sample based on the points in $\aset{Q}$ implicitly reflect the
covariance structure of $\aset{P}^s$.

We must ensure $\aset{Q}$ is a set of sufficiently many samples
from $\aset{P}^s$ after each cut: $\aset{P}^{s} \gets \aset{P}^{s-1}
\cap \aset{H}_{\vec{z}^{s}}$. To do so, we resample $2N$ points from
$\aset{P}^{s-1}$ using \textsc{hit-and-run}---half of these,
$\aset{R}$, are used to estimate the centroid $\vec{z}^s$ for the cut
and the other half, $\aset{S}$, are used to repopulate $\aset{Q}$
after the cut. Because $\aset{S}$ contains independent uniform samples
from $\aset{P}^{s-1}$, those in $\aset{P}^s$ after the cut constitute
independent uniform samples from $\aset{P}^s$ (rejection sampling). By
choosing $N$ sufficiently large,
we will have sufficiently many points to repopulate $\aset{Q}$.

Finally, we also need an initial set $\aset{Q}$ of uniform samples
from $\aset{P}^0$ but we only have a single point $\dpt^- \in
\xminus$.
The \textsc{RoundingBody} algorithm described
by \citet{sa_volume} uses $\bigOstar{D^4}$ membership queries to make
the convex body near-isotropic. We use this as a preprocessing step;
that is, given $\xminus$ and
$\dpt^- \in \xminus$ we make $\aset{P}^0 = \xminus \cap
\oneball{2R}{\dpt^-}$ and use the \textsc{RoundingBody} algorithm to
produce $\aset{Q} = \set{\dpt^j \in \aset{P}^0}$ for
Algorithm~\ref{alg:intersect}.

\subsubsection{Optimization over $\LP[1]$ Balls}
\label{sec:ballOpt}

Here we suggest improvements for \LP[1] minimization using iterative \textsc{IntersectSearch} and present them as \textsc{SetSearch} in
Algorithm~\ref{alg:setsearch}. 

First, since $\xtarget$, $\dpt^-$ and $\aset{Q}$ are the same for every iteration
of the optimization procedure, we only run the \textsc{RoundingBody}
procedure once as a preprocessing step. The set of samples
$\set{\dpt^j \in \aset{P}^0}$ it produces are sufficient to initialize
\textsc{IntersectSearch} at each stage of the binary search.  Second,
the separating hyperplane $\vec{h}_f^\vec{y}$ for point $\vec{y}$
given by Eq.~\eqref{eq:gradient}
is valid for all weighted \LP[1]-balls of cost $C <
\adCostFunc{\vec{y}}$. 
Thus, the final state from a successful call to
\textsc{IntersectSearch} can be used as the starting state for the
subsequent call to \textsc{IntersectSearch}.  

\eat{ First, since $\xtarget$, $\dpt^-$ and
  are the same for every iteration of the optimization procedure, the
  set $\aset{P}^0$ and the RoundingBody() procedure used to make it
  near-isotropic only needs to be estimated once. Second, suppose we
  are currently checking if $\xminus$ intersects with a ball of cost
  $C^t$.  Clearly, we can use any sample generated during this
  procedure that is in $\xminus$ and has cost less than $C^t$.
  Practically, when this occurs, we can note it and update $C^-$ to
  reflect the new minimum, however we would continue generating our
  samples for our current ply of the feasibility procedure as these
  can be reused later.  This is possible due to the following: any
  halfspace cuts generated while searching for valid instances that
  achieve cost $C^t$ are also valid cuts for any cost $C < C^t$ since
  the same $L1$-gradients also provide separating cuts for any lower
  cost ball. Thus, by saving the state of the last satisfied
  feasibility procedure, we can restart any subsequent feasibility
  procedure at that state without having to re-initialize it. This
  suggests the more efficient Algorithm~\ref{alg:setsearch}, which
  uses a variant of $\function{Intersecsearch}{}$ with a given initial
  convex body $\aset{P}^0$ and the estimation of its centroid
  $\vec{z}_0$ and covariance $V_0$.}
\begin{figure}
  \begin{minipage}{\linewidth}
  \begin{Algorithm}{Convex $\xminus$ Set Search}{alg:setsearch}
      \STATE $\function{SetSearch}{\aset{P}, \aset{Q}=\set{\dpt^j \in \aset{P}}, \maxCost, \minCost, \multGoal}$
      \WHILE{$\maxCost / \minCost > 1+\multGoal$}
        \STATE $C \gets \sqrt{\maxCost \cdot \minCost}$
        \STATE $[\dpt^\ast, \aset{P}^\prime, \aset{Q}^\prime] \gets \function{IntersectSearch}{\aset{P},\aset{Q},C}$
        \IF{intersection found}
          \STATE Let $\maxCost \gets \adCostFunc{\dpt^\ast}$
          \STATE $\aset{P} \gets \aset{P}^\prime$ and $\aset{Q} \gets \aset{Q}^\prime$
        \ELSE
        \STATE $C^+ \gets C$
        \ENDIF
      \ENDWHILE
      \STATE \textbf{Return:} $\dpt^\ast$
  \end{Algorithm}
  \end{minipage}
\end{figure}
\section{CONCLUSIONS \& FUTURE WORK}
\label{sec:discuss}

The analysis of our algorithms shows that $\classSpace^\mathrm{convex}$ is
\kSearchable\ for weighted \LP[1] costs. When the positive class is
convex we give efficient techniques that outperform 
previous reverse-engineering approaches for linear classifiers. When
the negative class is convex, we apply a randomized Ellipsoid method
to achieve efficient \kIMAC\ search. If the adversary is unaware
of which set is convex, they can trivially run both searches to
discover an \kIMAC\ with a combined polynomial query complexity.

Exploring near-optimal evasion is important for understanding
how an adversary may circumvent learners in security-sensitive
settings.
As described here, our algorithms may not always directly apply in practice
since 
%
%
various real-world obstacles persist. Queries may be only partially observable or noisy and
the feature set may be only partially known. Moreover,
an adversary may not be able to query all $\dpt \in
\xspace$. Queries must be objects (such as email) that
are mapped into $\xspace$. A real-world adversary must invert the
feature-mapping---a generally difficult task. These limitations necessitate
further research on the impact of partial observability and
approximate querying on \kIMAC\ search, and to design more secure filters.
Broader open problems include: is \kIMAC\ search possible on
other classes of learners such as SVMs (linear in a large possibly
infinite feature space)? Is \kIMAC\ search feasible against an
online learner that adapts as it is queried?
Can learners be made resilient to these threats and how does this
impact learning performance?

{\small
\bibliography{sources}
}

\end{document}